%% file: main.tex
\documentclass{article}
\usepackage{arxiv}

\usepackage[round]{natbib}

\usepackage[utf8]{inputenc} 
\usepackage[T1]{fontenc}    
\usepackage{hyperref}       
\usepackage{url}            
\usepackage{booktabs}       
\usepackage{amsfonts}       
\usepackage{nicefrac}       
\usepackage{microtype}      
\usepackage{amsmath,amsthm,amssymb}
\usepackage{mathrsfs}  
\usepackage{bm}
\usepackage{color}
\usepackage{subfigure}
\usepackage{dsfont}
\usepackage{blkarray, bigstrut}
\usepackage{graphicx}
\usepackage{tikz}
\usepackage{makecell}
\usepackage{wrapfig}
\usepackage{enumitem}
\usepackage{booktabs}
\usepackage{algorithm}
\usepackage{algorithmic}
\usepackage{bbm}
\usepackage{mdframed}
\usepackage{thmtools}
\usepackage{thm-restate}

\usepackage{parskip}

\setlist[enumerate]{itemsep=0.5pt, wide=\parindent}
\setlist[itemize]{itemsep=0.5pt, wide=\parindent}

\newcommand*\diff{\mathop{}\!\mathrm{d}}

\DeclareMathOperator{\err}{\sf err}

\DeclareMathOperator{\sign}{\sf sign}

\definecolor{Blue}{rgb}{0.0,0.0,1.0}
\newcommand{\todo}[1]{{\color{Blue}{$\mathsf{to \ do:}$ \emph{#1}}}}

\declaretheoremstyle[
    bodyfont=\normalfont\itshape,
    spacebelow=\parsep,
    spaceabove=\parsep,
          ]{examplestyle}

\providecommand{\qed}{\hfill $\blacksquare$}

\DeclareMathOperator{\diag}{\sf diag}

\DeclareMathOperator{\argmin}{\mathrm{arg}\min}

\DeclareMathOperator{\supp}{\sf supp}

\DeclareMathOperator{\ball}{\sf Ball}

\DeclareMathOperator{\tv}{\sf TV}

\DeclareMathOperator{\ot}{\sf OT}

\DeclareMathOperator{\proj}{\sf proj}

\providecommand{\inserttitle}{Classifier-independent Lower-Bounds for
  Adversarial Robustness}

 \title{\inserttitle}

\author{Elvis Dohmatob
  \email{e.dohmatob@criteo.com}\\\addr{Criteo AI Lab}
  }
 \date{\today}

\begin{document} 
 \maketitle

\input{00-abstract.tex}


\input{01-introduction.tex}
\input{02-contributions.tex}
\input{05-conclusion.tex}

\bibliographystyle{apalike}
\bibliography{literature.bib}

\input{06-appendix.tex}

\end{document}

%% file: 00-abstract.tex
\begin{abstract}
We theoretically analyse the limits of robustness to test-time adversarial and
noisy examples in classification. Our work focuses on deriving bounds which
uniformly apply to all classifiers (i.e all measurable functions from features
to labels) for a given problem. Our contributions are two-fold. (1) We use optimal transport theory to
derive variational formulae for the Bayes-optimal error a classifier can make on
a given classification problem, subject to adversarial attacks. The optimal
adversarial attack is then an optimal transport plan for a certain binary
cost-function induced by the specific attack model, and can be computed via a
simple algorithm based on maximal matching on bipartite graphs. (2) We derive
explicit lower-bounds on the Bayes-optimal error in the case of the popular
distance-based attacks. These bounds are universal in the sense that they depend
on the geometry of the class-conditional distributions of the data, but not on a
particular classifier. Our results are in sharp contrast with the existing
literature, wherein adversarial vulnerability of classifiers is derived as a
consequence of nonzero ordinary test error.
\end{abstract}

%% file: 01-introduction.tex
\section{Introduction}
\label{sec:intro}
\subsection{Context}
Despite their popularization, machine-learning powered systems
(assisted-driving, natural language processing, facial recognition, etc.) are
not likely to be deployed for critical tasks which require a stringent error margin, in a
closed-loop regime any time soon. One of the main blockers which has been
identified by practitioners and ML researchers alike is the phenomenon of
\emph{adversarial examples}~\cite{szegedy2013intriguing}. There is now an arms
race~\cite{athalye18} between adversarial attack developers and defenders, and
there is some speculation that adversarial examples in machine-learning might
simply be inevitable.

In a nutshell an adversarial (evasion) attack operates as follows. A classifier is trained
and deployed (e.g the road traffic sign recognition sub-system on an AI-assisted
car). At test / inference time, an attacker (aka adversary) may submit queries
to the classifier by sampling a data point $x$ with true label $y$, and
modifying it $x \rightarrow x^{\text{adv}}$ according to a prescribed threat
model. For example, modifying a few pixels on a road traffic sign
~\cite{onepixel}, modifying intensity of pixels by a limited amount determined
by a prescribed tolerance level 
variance per pixel, etc. 
  The goal of the attacker is to fool the classifier into classifying
  $x^{\text{adv}}$ with a label different from $y$.
A robust classifier tries to limit this failure mode, for a prescribed
  attack model.

In this manuscript, we establish universal lower-bounds
on the test error any classifier can attain under
adversarial attacks. 

\input{04-related.tex}

\subsection{Summary of our main contributions}
\input{summary.tex}

%% file: 04-related.tex
 Before announcing our contributions, we present
an overview a representative sample of this literature which is relevant to our
own contributions.
\subsection{Overview of related works}
Questions around adversarial examples and fundamental limits of defense
mechanisms, are an active area of research in machine-learning, with a large
body of scientific literature.
\label{sec:related-work}
\paragraph{Classifier-dependent lower-bounds.}
There is now a rich array of works which study adversarial examples as a natural
consequence of nonzero test error. In particular, let us mention
~\cite{tsipras18},
~\cite{schmidt2018},~\cite{goldstein},~\cite{gilmerspheres18},~\cite{saeed2018},~\cite{dohmatob19}.
These all use a form of the \emph{Gaussian isoperimetric
inequality}~\cite{boucheron2013}: in these theories, adversarial examples exist
as a consequence of ordinary test-error in high-dimensional problems with
concentrated class-conditional distributions. On such problems, for a classifier
which does not attain 100\% on clean test examples (which is likely to be the
case in practice), every test example will be close to a misclassified example,
i.e can be misclassified by adding a small perturbation. Still using Gaussian
isoperimetry, ~\cite{gilmernoise} has studied the relationship between
robustness to adversaries and robustness to random noise. The authors argued
that adversarial examples are a natural consequence of errors made by a
classifier on noisy images. 

One should also mention some works which exploit curvature of the decision
boundary of neural networks to exhibit the existence of vectors in
low-dimensional subspaces, which when added to every example in a target class,
can fool a classifier on a fraction of the samples~\cite{moosavi17,unipert}.
\paragraph{Universal / classifier-independent bounds.}
To our knowledge, ~\cite{fawzi18,saeed2018,bhagoji2019} are the only works to derive universal /
classifier-independent lower-bounds for adversarial robustness. Particularly,
~\cite{bhagoji2019} and ~\cite{pydi2019}, are the most related to ours. In
~\cite{bhagoji2019}, the authors
considered general adversarial attacks (i.e beyond distance-based models of
attack), and show that Bayes-optimal error for the resulting classification
problem under such adversaries is linked to a certain transport distance between
the class-conditional distributions (see our Theorem \ref{thm:main} for a
generalization of the result). This result is singularly different from the
previous literature as it applies even to classifiers which have zero test-error
in the normal / non-adversarial sense. Thus, there adversarial examples that
exist solely as a consequence of the geometry of the problem. The results in
section \ref{sec:variational} of our paper are strict extension of the
bounds in ~\cite{bhagoji2019}. The main idea in ~\cite{bhagoji2019,pydi2019} is to
construct an optimal-transport metric (w.r.t a certain binary cost-function
induced by the attack), and then use Kantorovich-Rubenstein duality to relate
this metric to the infimal adversarial error a classifier can attain under the
adversarial attack.
\ref{sec:variational}.

Finally, one should mention ~\cite{cranko19} which studies vulnerability of
hypothesis classes in connection to loss functions used. 

%% file: summary.tex
Our main contributions can be summarized as
follows.
\begin{itemize}
    \item 
    In section \ref{sec:variational}
      (after developing some background material in section
      \ref{sec:background}),
      we use optimal transport theory to derive variational formulae for the
      Bayes-optimal error (aka smallest possible test error) of a classifier
      under adversarial attack, as a function of the "budget" of the attacker,
        These formulae suggest that instead of doing adversarial
      training, practitioners should rather do normal training on adversarially
      augmented data. Incidentally, this is a well-known trick to boost up the
      adversarial robustness of classifiers to known attacks, and is usually
      used in practice
      under the umbrella name of "adversarial data-augmentation". See
      ~\cite{yao2020}, for example. In our manuscript, this principle appears as
      a natural consequence of our variational formulae.
      \item In section
      \ref{subsec:matching}, we also provide a realistic algorithm for computing
      the optimal universal attack plan via maximal matching of bipartite
      graphs, inspired by ~\cite{harel2015}.
      \item For the special case of
      distance-based attacks, we proceed in \ref{sec:distbased} to \emph{(1)}
      Establish universal lower-bounds on the adversarial Bayes-optimal error. These bounds are a
      consequence of concentration properties of light-tailed class-conditional
      distributions of the features (e.g sub-Gaussianity, etc.). \emph{(2)}
      Establish universal bounds under more general moment constraints
      conditions on the class-conditional distributions (e.g existence of
      covariance matrices for the class-conditional distributions
        of the features).
\end{itemize}                                                                   %

% ...

%% file: 02-contributions.tex
\section{Preliminaries}
\label{sec:background}
\input{010-notations.tex}


\section{Optimal transport characterization of adversarial vulnerability}
\label{sec:variational}
\subsection{Adversarial attacks as transport plans}
\label{sec:adversarial}
Given an attack model $\Omega$ on $\mathcal X$, meaning that $\Omega$ is a
closed subset of $\mathcal X^2$,
consider 
the binary cost-function $c_\Omega:\mathcal X^2 \rightarrow
\{0,1\}$ defined by
\begin{eqnarray}
c_{\Omega}(x,x')
  :=
  \begin{cases}0,&\mbox{ if }(x,x') \in  \Omega,\\1,&\mbox{ else.}
  \end{cases}
  \label{eq:cost}
  \end{eqnarray}
This cost-function is special in that, for every $(x,x') \in  \Omega$, one can
transport $x$ to $x'$ without incurring any cost at all. If $x$ and $x'$ happen
to belong to different classes, then an adversarial attack which replaces $x$
with $x'$ will be perfectly undetectable. As in ~\cite{bhagoji2019}, we start
with a variational formula for measuring the cost of a type-$\Omega$ for the
task of ``blunting'' the Bayes-optimal classifier for the classification problem
$(P^1,P^2)$.
\begin{restatable}[Adversarial total-variation]{df}{}
Let $\ot_{\Omega}(P^1,P^2)$ be the
optimal transport distance between $P^1$ and $P^2$ w.r.t to the ground cost
$c_{\Omega}$ defined in Eq. \eqref{eq:cost}, i.e
  \begin{eqnarray}
  \begin{split}
  \ot_{\Omega}(P^1,P^2)&:=\inf_{\gamma \in \Pi(P^1,P^2)}\int_{\mathcal
    X^2}c_{\Omega}(x_1,x_2)\diff\gamma(x_1,x_2)
    \\
    &=\inf_{\gamma \in \Pi(P^1,P^2)}\mathbb E_\gamma[c_{\Omega}(X_1,X_2)],
  \label{eq:otdef}
  \end{split}
  \end{eqnarray}
where $\Pi(P^1,P^2)$ is the set of all couplings of $P^1$ and $P^2$, i.e the set
of all measures on $\mathcal X^2$ with marginals $P^1$ and $P^2$,
and $(X_1,X_2)$ is a pair of r.vs on $\mathcal X$ with joint distribution $\gamma$.  
\end{restatable}
If $\gamma$ is a coupling of $P^1$ and $P^2$ and $(X_1,X_2) \sim \gamma$, with
abuse of language we shall also refer to $(X_1,X_2)$ as a coupling of $P^1$ and $P^2$.

\begin{restatable}{lm}{}
The ground-cost function  $c_\Omega$ is lower-semicontinuous (l.s.c) on $\mathcal X^2$.
\end{restatable}
\begin{proof}
  In fact, we proof that $c_\Omega$ is l.s.c iff $\Omega$ is closed in $\Omega$.
  Recall that the definition of lower-semicontinuity $c_\Omega$ is that the set
  $S_t := \{(x,x') \in \mathcal X^2 \mid c_\Omega(x,x') \le t\}$ is closed in
  $\mathcal X^2$ for every $t \in \mathbb R$. A simple calculation reveals that
  $$
  S_t = \begin{cases}\emptyset,&\mbox{ if }t < 0,\\\Omega,&\mbox{ if }0 \le t
    < 1,\\\mathcal X^2,&\mbox{ if }t \ge 1.\end{cases}
  $$
  Thus, $S_t$ is closed in $\mathcal X^2$ $\forall t \in \mathbb R$ iff $\Omega$
  is closed in $\mathcal X^2$.
\end{proof}

Thus, $\ot_\Omega(\cdot,\cdot)$ defines a distance over measures on the feature
space $\mathcal X$. In the particular case of distance-based attacks, we have
$\Omega = D_\varepsilon$ as defined in Eq. \eqref{eq:epsmodel}, and formula
\eqref{eq:otdef} can be equivalently written as
\begin{eqnarray}
  \ot_{D_\varepsilon} = \tv_\varepsilon(P^1,P^2) := \inf_{(X_1,X_2)}\mathbb
  P(d(X_1,X_2) > \varepsilon),
  \label{eq:perturbedvar}
\end{eqnarray}
where the infimum is taken over all couplings $(X_1,X_2)$ of $P^1$ and $P^2$. The
joint distribution $\gamma_\varepsilon$ of $(X_1,X_2)$ is then an optimal
adversarial attack plan for the classification problem $(P^1,P^2)$. Note that
the case $\varepsilon = 0$ conveniently corresponds to the usual definition of
total-variation, namely
\begin{eqnarray}
\begin{split}
  \tv(P^1,P^2) :&= \sup_{A \subseteq \mathcal X \text{
  measurable}}P^1(A)-P^2(A)= \inf_{(X_1,X_2)}\mathbb P(X_1 \ne X_2),
  \end{split}
  \label{eq:tvdef}
\end{eqnarray}
The RHS of the above formula is usually referred to as \emph{Strassen's} formula
for total-variation.

Coincidentally, the metric $\tv_\varepsilon$ in \eqref{eq:perturbedvar} has
been studied in context of statistical testing, under the name "perturbed
variation"~\cite{harel2015} as robust replace for usual total-variation.
Moreover, the authors proposed an efficient algorithm for computing both the
optimal plan $\gamma_\varepsilon$ as a maximal graph matching in a bipartite
graph. In has also been studied in
~\cite{yao2020} in the context of adversarial attacks.

\paragraph{Link to classical theory of classification.}
It is well-known~\cite{ReidW11} in standard classification theory that the
Bayes-optimal error is exactly equal to
\begin{eqnarray}
  \err^*(P^1,P^2) := \frac{1}{2}(1-\tv(P^1,P^2)).
  \label{eq:classical}
\end{eqnarray}
Thus, one might expect that the adversarial total-variation metric
$\tv_\Omega(\cdot,\cdot)$ defined in Eq. \eqref{eq:otdef} would play a role in
control of the adversarial Bayes-optimal error $\err_\Omega^*(\cdot,\cdot)$
(defined in Eq. \eqref{eq:risk}) which is similar to the role played by ordinary
total-variation $\tv(\cdot,\cdot)$ (defined in Eq. \eqref{eq:tvdef}) plays in
formula \eqref{eq:classical} for the classical / standard Bayes-optimal error.
This is indeed the case.

  \begin{restatable}{rmk}{}
    The adversarial Bayes-optimal error $\err^*_\Omega(P^1,P^2)$ under type-$\Omega$
    adversarial attacks should not be confused with the adversarial error of the
    standard / classical Bayes-optimal classifier $h^* := \argmin_h
    \err(P^1,P^2)$ for the unattacked classification problem.
    In fact $\err^*_\Omega(P^1,P^2) \le \err^*_\Omega(h^*; P^1, P^2)$, and we
    can construct explicit scenarios in which the inequality is strict (e.g
    one-dimensional classification problem whose class-conditional distributions are
    gaussians with different means and same variance, under the distance-based
    attack model $\Omega := \{(x,x') \in \mathbb R^2 \mid |x-x'| \le \varepsilon\}$).
  \end{restatable}

\begin{restatable}[Extension of Theorem 1 of ~\cite{bhagoji2019}]{prop}{main}
For any attack model $\Omega$ on the feature space $\mathcal X$, the
adversarial Bayes-optimal error under type-$\Omega$ attacks 
 \begin{eqnarray}
   \err_{\Omega}^*(P^1,P^2) \ge \frac{1}{2}(1-\ot_{\Omega}(P^1,P^2)).
 \end{eqnarray}
\label{thm:main}
\end{restatable}
Note that the reverse inequality does not hold in general. A remarkable
exception is the case of distance-based attacks with a distance $d$ that turns the
 feature space $\mathcal X$ into a \emph{complete separable} metric space with
 the \emph{midpoint property}\footnote{That is, for every $x,x' \in \mathcal X$,
   there exists $z \in \mathcal X$ such that $d(x,z) = d(x',z)=d(x,x')/2$.}.
 Examples of such spaces include complete riemannian manifolds and any closed
 convex subset of a separable Banach space. We shall return to such spaces in
 section \ref{subsec:augment}.
\subsection{Characterizing the adversarial error via optimal transport}
Henceforth, assume the feature space $\mathcal X$ is
\emph{Polish} (i.e $\mathcal X$ is metrizable, complete, and
separable).
 Finally, given a subset $U \subseteq \mathcal
X$, define its $\Omega$-closure $\overline{U}_{\Omega}$ by
\begin{eqnarray}
\overline{U}_{\Omega} := \{x \in \mathcal X \mid (x,x') \in \Omega \text{ for
  some }x' \in U\}.
  \label{eq:Uneb}
\end{eqnarray}
In the case of metric attacks where $\Omega = D_\varepsilon := \{(x,x') \in
\mathcal X^2 \mid d(x,x') \le \varepsilon\}$, we have $\overline{U}_\Omega =
U^{\varepsilon}$, where $U^\varepsilon$ is the $\varepsilon$-neighborhood of $U$
defined by
 \begin{eqnarray}
   U^\varepsilon := \{x \in \mathcal X \mid d(x,x') \le \varepsilon\text{ for
   some }x' \in U\}.
 \end{eqnarray}

The following theorem is a direct application of \emph{Strassen's
  Marriage Theorem} (see  ~\cite[Theorem 1.27]{villaniTopics}), and is as a
first simplification of the complicated distance $\ot_{\Omega}$ that appears in
Proposition \ref{thm:main}. For distance-based attacks has been, a special case
of our result has been independently obtained in ~\cite{pydi2019}.

 \begin{restatable}{thm}{strassen}
   Let $\Omega$ be an attack model on $\mathcal X$. Then we have the identity
 \begin{eqnarray}
 \ot_{\Omega}(P^1,P^2) = \underset{ U \subseteq \mathcal
   X\;\text{closed}}{\sup}\; P^1(U)-P^2(\overline{U}_{\Omega}).
 \end{eqnarray}
 In particular, for distance-based attacks we have
 $
\ot_{\varepsilon}(P^1,P^2) = \underset{U \subseteq \mathcal X\text{
    closed}}{\sup}\;P^2(U)-P^1(U^{\varepsilon}).
$
 \label{thm:strassen}
\end{restatable}

We now present a lemma which allows us to rewrite the optimal transport distance
$\ot_\Omega$ as a linear program over partial transport plans will be one of the
main ingredients in the proof of Thm \ref{thm:badass} below. The lemma is
important in its own right.
\begin{restatable}{lm}{partial}
  Let $\Omega$ be an attack model on the feature space $\mathcal X$. Then
  \begin{eqnarray}
    \tv_\Omega(P^1,P^2) = \inf_{\gamma \in \Pi_{\le} (P^1, P^2),\;\supp(\gamma)
    \subseteq \Omega}1-\gamma(\mathcal X^2),
  \end{eqnarray}
  where $\Pi_{\le} (P^1, P^2)$ is the set of partial couplings of $P^1$ and
  $P^2$, i.e Borel measures on $\mathcal X^2$ whose marginals are
  dominated by the $P^k$'s.
  \label{lm:partial}
\end{restatable}
\begin{proof}
Define  the quantity
$$
E(P^1,P^2):=\inf_{\gamma\in \Pi_{\le}(P^1,P^2),\; \supp(\gamma)\subseteq
  \Omega} 1 - \gamma(\mathcal{X}^2),
$$
where $\Pi_{\le}$ denotes the set of partial transport plans, i.e. probabilities
on $\mathcal{X}^2$ with marginals smaller than $P^1$ and $P^2$ respectively.
First, let us show that $\text{OT}_{\Omega}=E$. Let $\gamma \in
\Pi(P^1,P^2)$ and let $\tilde \gamma$ be its restriction to $\varepsilon'$.
Then $\tilde \gamma$ is feasible for $E$ and it holds
$\gamma(\Omega)=1-\tilde\gamma(\mathcal{X}^2)$ so $E\leq \text{OT}_{\Omega}$.
Conversely, let $\gamma$ be feasible for $E$ and consider any $\tilde \gamma\in
\Pi(P^1-\proj^1_\#\gamma, P^2-\proj^2_\#\gamma)$. Then $\gamma +\tilde
\gamma$ is feasible for $\text{OT}_{\Omega}$ and $(\gamma+\tilde
\gamma)(\Omega)= \tilde \gamma(\Omega)\leq \tilde \gamma(\mathcal{X}^2) =
1-\gamma(\mathcal{X}^2)$. So, $\text{OT}_{\Omega}\leq E$ and thus
$\text{OT}_{\Omega}=E$.
\end{proof}

\subsection{Adversarial couplings}
\label{subsec:augment}
We now turn to distance-based attacks and refine representation presented in the
previous lemma. Recall that a metric space is said to have the midpoint property
if for every pair of points $z$ and $z'$, there is a point $\eta(z,z')$ in the
space which seats exactly halfway between them. Examples of such spaces include
normed vector-spaces and riemannian manifolds. For our next result, it will be
important to be able to select the midpoint $\eta(z,z')$ in a measurable manner
almost-everywhere.
\begin{restatable}[Measurable Midpoint (MM) property]{cond}{mmp}
 A metric space $\mathcal Z = (\mathcal Z,d)$ is said to satify the
 measurable midpoint (MM) property if for every Borel measure $Q$ on $\mathcal
 Z^2$ there exists a $Q$-measurable map $\eta: \mathcal Z^2 \to \mathcal Z$ such
 that $d(z,\eta(z,z')) = d(z',\eta(z,z')) = d(z,z')/2$ for all $z,z' \in \mathcal Z$.
  \label{cond:mmp}
\end{restatable}
The feature space for most problems in machine learning together with the
distances usually used in adversarial attacks, satisfies the measurable midpoint
property \ref{cond:mmp}. Indeed, generic examples of metric spaces
which satisfy this condition include:
\begin{itemize}
  \item Hilbert spaces.
  \item Closed convex subsets of Banach spaces.
  \item Complete riemannian manifolds (equipped
    with the geodesic distance).
  \item Complete separable metric spaces with the midpoint property.
  \end{itemize}
In fact, in the first two examples, the midpoint mapping $\eta$ can be chosen to
be continuous everywhere. The last example, which can be proved via the classical
\emph{Kuratowski-Ryll-Nardzewski measurable selection theorem}, is the most
general and most remarkable, and deserves an explicit restatement.
\begin{restatable}{lm}{csmp}
  Every complete separable metric space which has the midpoint property also has
  the measurable midpoint property.
\end{restatable}

The following theorem, which is proved in the appendix (as are all the other
theorems in this manuscript), is one of our main results. 

\begin{restatable}[Adversarially augmented data, a proxy for adversarial robustness]{thm}{badass}
  \label{thm:badass}
Consider a classification problem $(P^1,P^2)$. Suppose $d$ is a distance on the
feature space $\mathcal X$ with the MM property \ref{cond:mmp}, and consider the
distance-based attack model $D_\varepsilon := \{(x,x') \in \mathcal X^2 \mid
d(x,x') \le \varepsilon\}$.  Recall the
   definition of $\tv_\varepsilon(P^1,P^2)$ from Eq.
   ~\eqref{eq:perturbedvar}.
  Define 
    \begin{eqnarray}
      \begin{split}
        \widetilde{\tv}_{\varepsilon}(P^1,P^2) &:= \inf_{\gamma_1,\gamma_2}
        \tv({\proj^2}_{\#}\gamma_1,{\proj^1}_{\#}\gamma_2),\\        
        \widetilde{\widetilde{\tv}}_{\varepsilon}(P^1,P^2) &:= \inf_{a_1,a_2
          \;\text{type-}D_{\varepsilon/2}}\tv({a_1}_{\#}P^2,{a_2}_{\#}P^1),
     \end{split}
    \end{eqnarray}
    where "\#" denotes pushfoward of measures and the 1st inf. is taken
   over all pairs of distributions $(\gamma_1,\gamma_2)$ on $\mathcal X^2$ concentrated on
   $D_{\varepsilon/2}$ 
   such that
   ${\proj^1}_{\#}\gamma_1=P^2$ and ${\proj^2}_{\#}\gamma_2=P^1$. It holds that
   \begin{eqnarray}
     \tv_{\varepsilon}(P^1,P^2) = \widetilde{\tv}_{{\varepsilon}}(P^1,P^2) \le
     \widetilde{\widetilde{\tv}}_{{\varepsilon}}(P^1,P^2),
     \label{eq:great}
   \end{eqnarray}
   and there is equality if $P^1$ and $P^2$ have densities w.r.t the Borel
   measure on $\mathcal X$.

Consequently, we have the following lower-bound for the adversarial Bayes-optimal error:
  \begin{eqnarray}
  \begin{split}
\err_{\varepsilon}^*(P^1,P^2) &\ge \frac{1}{2}(1-\tv_{\varepsilon}(P^1,P^2)) \frac{1}{2}(1-\widetilde{\tv}_{\varepsilon}(P^1,P^2)) \ge                                           \frac{1}{2}(1-\widetilde{\widetilde{\tv}}_{\varepsilon}(P^1,P^2)).           \end{split}
   \label{eq:great}
 \end{eqnarray}
\label{thm:great}
\end{restatable}

\subsection{Case study: (separable) Banach spaces}
Thm. \ref{thm:badass} has several important consequences, which will be
heavily explored in the sequel. A particularly simple consequence is the
Consider the special case where $\mathcal X=(\mathcal X,\|\cdot\|)$, a
separable Banach space 
Given a point $z \in \mathcal X$, let $P^1+z$ be the translation of $P^1$ by
$z$.
For $z,z' \in \ball_{\mathcal X}(0;\varepsilon/2)$, consider the type-$D_{\varepsilon/2}$
distance-based attacks $a^{z,z'}_1,a_2^{z,z'}:\mathcal X \rightarrow \mathcal X$ defined by
$a^{z,z'}_1(x)=x-z$ and $a^{z,z'}_{2}(x)=x+z'$. One computes
\begin{eqnarray*}
\begin{split}
  \tv_{\varepsilon}(P^1,P^2) &:= \inf_{a_1,a_2\text{ type-}D_{\varepsilon/2}}\tv({a_1}_{\#}P^1,{a_2}_{\#}P^2)
  \\
  &\le \inf_{\|z\| \le \varepsilon/2,\;\|z'\| \le
    \varepsilon/2}\tv({a^{z,z'}_1}_{\#}P^2,{a^{z,z'}_{2}}_{\#}P^1)
    \\  
  &= \inf_{\|z\| \le \varepsilon/2,\;\|z'\| \le
    \varepsilon/2}\tv(P^1 - z,P^2 + z')\\
&\le \inf_{\|z\| \le \varepsilon}\tv(P^1,P^2 + z),
\end{split}
\end{eqnarray*}
where $P^2 + z$ is the translation of distribution $P^2$ by the vector $z$. Note
that in the above upper bound, the LHS can be made very concrete in case the
distributions are prototypical (e.g multivariate Gaussians with same covariance
matrix; etc.). Thus we have the following result
\begin{restatable}{cor}{tool}
Let the feature space $\mathcal X$ be a normed vector space and consider a
distance-based attack model $D_\varepsilon = \{(x,x') \in \mathcal X^2 \mid \|x'-x\|
\le \varepsilon\}$. Then, it holds that
\begin{eqnarray}
\err^*_{\varepsilon}(P^1,P^2) \ge \frac{1}{2}\left(1-\sup_{\|z\| \le
  \varepsilon}\tv(P^1,P^2+z)\right).
  \label{eq:z}
\end{eqnarray}
\label{thm:tool}
\end{restatable}
A solution in $z^*$ to optimization problem
in the RHS of \eqref{eq:z} would be a (\emph{doubly}) universal adversarial perturbation: a
single fixed small vector which fools all classifiers on proportion of test
samples. Such a phenomenon has been reported in ~\cite{moosavi17}.

\subsection{Computing the optimal attack plan}
\label{subsec:matching}
It turns out that the optimal transport plan $\gamma_{\Omega}$ which realizes
the distance $\tv_{\Omega}(P^1,P^2)$ in Thm. \ref{thm:great} can be
efficiently computed via matching (graph theory), by using iid samples from both
distributions. The recent work ~\cite{yao2020} has studied a metric on
propbability measures which coincidentally corresponds to the metric
$\tv_\varepsilon(\cdot,\cdot)$ we defined in Eq. \ref{eq:perturbedvar}.
This metric even goes back to the authors of ~\cite{harel2015}, who proposed it
under the name of ``perturbed variation'', for the purposes of robust
statistical hypothesis testing.

The following proposition is an adaptation of ~\cite{harel2015}, and the proof is
similar and therefore omitted.
\begin{restatable}[Optimal universal attacks via maximal matching]{prop}{}
Suppose $P^1 = \sum_{i=1}^{n_1}\mu^1_i\delta_{x^1_i}$
and $P^2 = \sum_{j=1}^{n^2}\mu^2_j\delta_{x^2_j}$ are distributions with finite
supports $V^1 := \{x^1_1,\ldots,x^1_{n_1}\} \subseteq \mathcal X$, $V^2 =
\{x^2_1,\ldots,x^2_{n_2}\} \subseteq \mathcal X$, with weights
$(\mu^1_1,\ldots,\mu^1_{n_1}) \in \Delta_{n_1}$, $(\mu^2_1,\ldots,\mu^2_{n_2})
\in \Delta_{n_2}$. Let $G$ be the bipartite graph with vertices $V^1 \cup V^2$
and edges $(V^1 \times V^2) \cap \Omega$. Then we can compute
$\ot_{\Omega}(P^1,P^2)$ via the following linear program:
\begin{eqnarray}
\begin{split}
   &\ot_{\Omega}(P^1,P^2)=\min_{w^1,w^2,\gamma}
   \sum_{i=1}^{n_1}w^1_i + \sum_{j=1}^{n_2}w^2_j\\
  \mathrm{Subject}\;\mathrm{to}\;&w^1 \in \mathbb R^{n_1}_+, w^2 \in \mathbb
  R^{n_2}_+, \gamma \in \mathbb R^{n_1 \times n_2}_+\;
    \gamma_{i,j} = 0\;\forall (i,j) \not \in E\\
    &\sum_{x^2_j \sim x^1_i}\gamma_{i,j} + w^1_i = \mu^1_i\;
    \sum_{x^1_i \sim x^2_j}\gamma_{i,j} + w^2_j = \mu^2_j,\;\forall i \in [\![n_1]\!],\;j \in [\![n_2]\!].
\end{split}
\label{eq:combi}
\end{eqnarray}
Moreover, the optimal transport plan $\gamma_\Omega$ to the above LP, and the
can be computed in
$\mathcal O(k\min(n_1,n_2)\sqrt{\max(n_1,n_2)}$ time, where $k$ is the average
number of pairs edges in the graph $G$.
\label{thm:matching}
\end{restatable}
A simple algorithm for computing optimal matching $\gamma$ is given in
Alg. \ref{alg:pv} below. This algorithm is an adaptation of the algorithm in
~\cite{harel2015}, for computing their `perturbed variation'', a robust version of
total-variation which corresponds to our $\tv_{\varepsilon}(P^1,P^2)$ defined in
\eqref{eq:perturbedvar}. Also see  ~\cite{yao2020} and ~\cite{metric2010} for
related work.
\begin{algorithm}[H]
   \caption{Empirical approx. of $\tv_\Omega(P^1,P^2)$ and optimal attack
     plan $\gamma_\Omega$, for an attack model $\Omega$}
   \label{alg:pv}
\begin{algorithmic}
   \STATE {\bfseries Input:} $V^k = \{x_1^k,\ldots,x_{n_1}^k\}$, where
   $x_1^k,\ldots,x_{n_1}^k \sim P^k$ is a sample of size $n_k$ for $k \in \{1,2\}$.
   \STATE {\bfseries Construct} $G$ a bipartite graph  $G$ with vertex set $V^1 \cup V^2$ and
   edges  $(V^1 \times V^2) \cap \Omega$.
   \STATE {\bfseries Compute} maximal matching on $\gamma_\Omega$ on $G$.
   \STATE {\bfseries Return} $\gamma_\Omega$ and $\frac{1}{2}(\frac{u_1}{n_1} + \frac{u_2}{n_2})$,
   where $u_k$ is the number of unmatched vertices in $V^k$.
\end{algorithmic}
\end{algorithm}

\begin{restatable}[No Free lunch for the attacker]{rmk}{}
Unforturnately for the attacker, convergence of the above algorithm (or any
other algorithm) for computing $\tv_\Omega(P^1,P^2)$ from samples will
typically suffer from the curse of dimensionality. For example, in the case of
distance-based attacks on $\mathbb R^m$, this remark is a direct consequence of
~\cite[Theorems 3 and 4]{harel2015} where it is shown that the sample complexity
is exponential in the dimensionality $m$.
\end{restatable}

\section{Universal bounds for general distance-based attacks}
\label{sec:distbased}

We now turn to the special case distance-based attacks on a metricized feature
space $\mathcal X = (\mathcal X,d)$. We will exploit geometric properties of the
class-conditional distributions $P^k$ to obtain upper-bounds on
$\tv_\varepsilon(P^1,P^2)$, which will in turn imply lower lower bounds on
optimal error (thanks to Thm. \ref{thm:great}. 

\subsection{Bounds for light-tailed class-conditional distributions}
We now establish a series of upper-bounds on $\tv_\varepsilon(P^1,P^2)$, which
in turn provide hard lower bounds for the adversarial robustness error on any
classifier for the binary classification experiment $(P^1,P^2)$, namely
$\err^*_\varepsilon(P^1,P^2)$. These bounds are a consequence of ligh-tailed
class conditional distributions.

{{\bfseries Name of the game.}} We always have the upper-bound
$\err^*_\varepsilon(P^1,P^2) \le 1/2$ (attained by random guessing). Thus, the
real challenge is to show that $\err^*_\varepsilon(P^1,P^2) = 1/2 +
o_\varepsilon(1)$, 
where
$o_\varepsilon(1)$
goes to zero as the attack "budget" $\varepsilon$ is increased.

\begin{restatable}[Bounded tails]{df}{}
Let $\alpha:[0,\infty) \rightarrow [0, 1]$ be a function. We say the a
distribution $Q$ on $(\mathcal X,d)$ has $\alpha$-light tail about the point
$x_0 \in \mathcal X$ if $\mathbb P_{x \sim Q}(d(x,x_0) > t) \le
\alpha(t)\;\forall t \ge 0$.
\end{restatable}

\begin{restatable}[The curse of light-tailed class-conditional distributions]{thm}{lighttail}
  Suppose $P^1$ and $P^2$ have $\alpha$-light tails about a points $\mu_1 \in
  \mathcal X$ and $\mu_2 \in \mathcal X$ resp. Then
\begin{eqnarray}
\err^*_\varepsilon(P^1,P^2) \ge 1/2-\alpha\left((\varepsilon-d(\mu_1,\mu_2))/2\right),
\end{eqnarray}
holds for every $\varepsilon \ge d(\mu_1,\mu_2)$.
\label{thm:lighttail}
\end{restatable}
\begin{proof}[Proof of Theorem \ref{thm:lighttail}]
Define $\widetilde{\varepsilon}:=(\varepsilon-d(\mu_1,\mu_2))/2$ and let $(X_1,X_2)$
be a any coupling of $P^1$ and $P^2$. By definition of
$\tv_\varepsilon(P^1,P^2)$, we have
\begin{eqnarray*}
\begin{split}
\tv_\varepsilon(P^1,P^2) &\le \mathbb P(d(X_1,X_2) > \varepsilon) \le \mathbb P(d(X_1,\mu_1) + d(X_2,\mu_2) > \varepsilon - d(\mu_1,\mu_2)) \\
&\le \mathbb P\left(d(X_1,\mu_1) > \widetilde{\varepsilon}\right) + \mathbb
P\left(d(X_2,\mu_2) > \widetilde{\varepsilon}\right) \le
\alpha(\widetilde{\varepsilon})+\alpha(\widetilde{\varepsilon})=2\alpha(\widetilde{\varepsilon}),
\end{split}
\end{eqnarray*}
where the 1st inequality is the triangle inequality and the 2nd is a union
bound. The result then follows by minimizing over the coupling $(X_1,X_2)$.
\end{proof}

As an example, take $\mathcal X=(\mathbb R^m,\|\cdot\|_\infty)$ and $P^k =
\mathcal N(\mu_k,\sigma^2 I_m)$. Then for every $t \ge 0$, one computes 
\begin{eqnarray*}
\begin{split} 
\mathbb P(d(x,\mu_1) > t) &= \mathbb P(\|N(0,\sigma^2 I_m)\|_\infty > t)
\le m\mathbb P(|\mathcal N(0,\sigma^2)| > t)\le 2me^{-t^2/(2\sigma^2)}.
\end{split}
\end{eqnarray*}
Thus, we can take $\alpha(t) = 2me^{-\frac{t^2}{2\sigma^2}}\;\forall t\ge 0$ and
obtain that the test error of any classifier under $\ell_\infty$-norm
adversarial attacks of size $\varepsilon$ is 
$\ge \err^*_\varepsilon(P^1,P^2) \ge (1-\tv_\varepsilon(P^1,P^2)) /2 \ge
1/2 - me^{(\varepsilon-\|\mu_1-\mu_2\|_\infty)^2/(4\sigma^2)} = 1/2 + o(1),
$
which increases to $1/2$, i.e the performance of random guessing, exponentially
fast as $\varepsilon$ is increased. This is just another manifestation of the
concentration of measure in high-dimensions (large $m$), for distributions which
are sufficiently ``curved''.

\subsection{Bounds under general moment and tail constraints on the
  data distribution}
The following condition will be central for the rest of the manuscript.

\begin{restatable}[Moment constraints]{cond}{}
 There exists $\alpha > 0$ and $M:\mathbb R_+ \rightarrow \mathbb R_+$ be an increasing
 convex function such that $M(0)=0$. We will
 occasionally assume that there exist $\mu_1,\mu_2 \in \mathcal X$ such that the following
 moment condition is satisfied
\begin{equation}
\begin{split}
\underset{x_1 \sim P^1}{\mathbb E}[M(d(x_1,\mu_1))] + \underset{x_2 \sim P^2}{\mathbb E}[M(d(x_2,\mu_2))] \le 2\alpha.
\end{split}
\label{eq:orlicz}
\end{equation}

\label{cond:orlize}
\end{restatable}

  For example, if each $P^k$ is $\sigma$-subGaussian about $\mu_k \in \mathbb
  R^m$, then we may take $M(r):=e^{r^2/\sigma^2}-1$ to satisfy the condition.
  More generally, recall that the Orlicz $M$-norm of a random variable $X_k \sim
  P^k$ (relative to the reference point $\mu_k$) is defined by
\begin{eqnarray}
\|X_k\|_M := \inf\{C > 0 \mid \mathbb E[M(d(X_k,\mu_k)/C)] \le 1\}.
\end{eqnarray}
Thus, Condition \ref{cond:orlize} is more general than
demanding that both $P^1$ and $P^2$ have Orlicz $M$-norm at most $\alpha$.
\begin{restatable}[The curse of bounded moments]{thm}{metricbound}
Suppose $(P^1,P^2)$ satisfies Condition \ref{cond:orlize}. Then
 \begin{eqnarray}
\err^*_\varepsilon(P^1,P^2) \ge 1/2(1- \alpha/M(\widetilde{\varepsilon})),\;\forall \varepsilon \ge 0.
 \end{eqnarray}
\label{thm:metricbound}
\end{restatable}
\begin{proof}
  Define $\widetilde{\varepsilon}:=(\varepsilon-d(\mu_1,\mu_2))/2$.
  Let $(X_1,X_2)$ be a coupling of $P^1$ and $P^2$. Then, by the definition of
  $\tv_\varepsilon(P^1,P^2)$, we have
 \begin{eqnarray*}
 \begin{split}
 \tv_\varepsilon(P^1,P^2) &\le \mathbb P(d(X_1,X_2) > \varepsilon)\le \mathbb P\left(d(X_1,\mu_1) + d(X_2,\mu_2) > \varepsilon - d(\mu_1,\mu_2)\right)\\
&\le \mathbb P\left(M\left(\frac{d(X_1,\mu_1)}{2} + \frac{d(X_2,\mu_2))}{2}\right) > M(\widetilde{\varepsilon})\right)\\
&\le \mathbb P\left(\frac{M(d(X_1,\mu_1))+M(d(X_2,\mu_2))}{2} > M(\widetilde{\varepsilon})\right)\\
&\le \alpha/M(\widetilde{\varepsilon}),
 \end{split}
 \end{eqnarray*}
 where the 2nd inequality is the triangle inequality; the 3rd inequality is
 because $M$ is increasing; the 4th is because $M$ is convex; the 5th is Markov's inequality and the moment the assumption.
\end{proof}

A variety of corollaries to Thm. \ref{thm:metricbound}
can be obtained by considering different choices for the moment function $M$ and
the parameter $\alpha$. More are presented in the supplementary materials. For
example if $P^1$ and $P^2$ have $d(\cdot,x_0) \in L^p(P^1) \cap L^p(P^2)$ for
some (and therefore all) 
$x_0 \in \mathcal X$, we may take $M(r) := r^p$ and $\alpha
= W_{d,p}(P^1,P^2)^p$, where $W_{d,p}(P^1,P^2)$ is the order-$p$ Wasserstein
distance between $P^1$ and $P^2$, and obtain the following corollary, which was
also obtained independently in the recent paper ~\cite{pydi2019}. We have

\begin{restatable}[Lower-bound from Wasserstein distance]{cor}{wassbound}
Under the conditions in the previous paragraph, we have
\begin{eqnarray}
\err^*_\varepsilon(P^1,P^2) \ge 
  \frac{1}{2}\left(1-\left(\frac{W_{d,p}(P^1,P^2)}{\varepsilon}\right)^p\right).
\end{eqnarray}
\end{restatable}
\begin{proof}
  Follows from Theorem \ref{thm:metricbound} with $\mu_1=\mu_2 = x_0 \in
  \mathcal X$ (any point!), $M(r) \equiv (2r)^p$, $\alpha
  = W_{d,p}(P^1,P^2)^p$.
\end{proof}

\input{comparison.tex}

%% file: 010-notations.tex
\subsection{Classification framework}
\label{subsec:notations}

All through this manuscript, the feature space will be denoted $\mathcal X$.
The \emph{label (aka classification target)} is a random variable $Y$ with
values in $\mathcal Y = \{1,2\}$, and random variable $X$ called the
\textit{features}, with values in $\mathcal X$. We only consider binary
classification problems in this work. The goal is to predict $Y$ given $X$. This
corresponds to prescribing a measurable function $h:\mathcal X \rightarrow
\{1,2\}$, called a \emph{classifier}. The joint distribution $P_{X,Y}$ of $(X,Y)$ is
unkown. The goal of learning is to find a classifier $h$ (e.g a deep neural net)
such that $h(X) = Y$ as often as possible, possibly under additional constraints.

  In this work, as in ~\cite{bhagoji2019},
we will only consider \emph{balanced} binary classification problems, 
where the labels are \emph{equiprobable} , i.e $\mathbb P(Y=1)=\mathbb P(Y=2)=1/2$.
Multiclass problems
can be considered in one-versus-all fashion. For each label $k \in \{1, 2\}$, we
define the (unnormalized) probability measure $P^k$ on the feature space $\mathcal X$ by
\begin{eqnarray}
  P^k(A) :=\mathbb P(X \in A, Y=k) = 
  \frac{1}{2}\mathbb P(X \in A|Y=k),
  \label{eq:pk}
\end{eqnarray}
for every measurable $A \subseteq X$.
Thus, $P^k$ is an unnormalized probability distribution on the feature space
$\mathcal X$ which integrates to $1/2$, and the classification problem is
therefore entirely captured by the pair $P=(P^1,P^2)$, also called a
\emph{binary experiment}~\cite{ReidW11}.

The notions of metric and pseudo-metric spaces will come of often in the manuscript.
\begin{restatable}[Metric and pseudo-metric spaces]{df}{}
  \label{df:metricspace}
A mapping $d:\mathcal X^2 \to \mathcal X$ is called a pseudo-metric on $\mathcal
X$ iff for all $x,x',z \in \mathcal X$, the following hold:
\begin{itemize}
\item {\bfseries Reflexivity:} $d(x,x) = 0$.
\item {\bfseries Symmetry: }$d(x,x') = d(x',x)$.
\item {\bfseries Triangle inequality:} $d(x,x') \le d(x,z) + d(z,x')$.
\end{itemize}
The pair $(\mathcal X,d)$ is then called a pseudo-metric space.
If in addition, $d(x,x') = 0 \implies x = x'$, then we say $d$ is a metric (or
distance) on $\mathcal X$, and the pair $(\mathcal X, d)$ is called a metric space.
\end{restatable}

\input{adversarial_notations.tex}

%% file: adversarial_notations.tex
\subsection{Models of adversarial attack}
\label{subsec:adv_notations}
In full generality, an \emph{adversarial attack model} on the feature space $\mathcal
X$ (a topological space) is any closed subset $\Omega \subseteq \mathcal X^2$.
Given points $x',x \in \mathcal X$, we call $x'$ an \emph{adversarial example} of $x$ if
$(x,x') \in \Omega$. The subset $\diag(\mathcal X^2) := \{(x,x) \mid x
\in \mathcal X\}$ corresponds to classical / standard classification theory
where there is no adversary.
\begin{wrapfigure}{r}{0.33\textwidth}
  \centering
  \begin{tikzpicture}[scale=.8]
    \node at (-2.3, 0) {\large $\mathcal X$};
    \node at (0, 2.3) {\large $\mathcal X$};        
\draw[very thick] (2,2)--(-2,2)--(-2,-2)--(2,-2)--cycle;
\fill[color=green,draw=black] (-1,2)--(-2,2)--(-2,1)--(1,-2)--(2,-2)--(2,-1)--cycle;

\draw[thick,dashed](-2,2)--(2,-2);
\draw[thick,<->] (-0.5,-0.5)--node[above]{$\varepsilon$}(0.5,0.5);
\end{tikzpicture}
\caption{Showing a generic distance-based attack model. The green region
  corresponds to the set $D_\varepsilon \subseteq \mathcal X^2$ defined in
  \eqref{eq:epsmodel}. An attacker is allowed to swap any point $x \in \mathcal X$ with
  another point $x' \in \mathcal X$, if the pair $(x,x')$ lies in the green
  region.}
\end{wrapfigure}
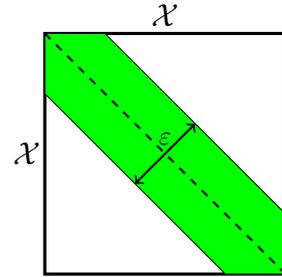
A nontrivial example is the case of so-called
distance-based attacks, where $d$ is a metric on $\mathcal X$ and the attack model
$\Omega  = D_\varepsilon$, where
\begin{eqnarray}
\label{eq:epsmodel}
D_\varepsilon = \{(x,x') \in \mathcal X^2 \mid d(x,x') \le \varepsilon\},
\end{eqnarray}
with $\varepsilon \ge 0$ being the budget of the attacker.
These include the well-known $\ell_p$-norm attacks in finite-dimensional
euclidean spaces usually studied in the literature (e.g
~\cite{szegedy2013intriguing,tsipras18,schmidt2018,goldstein,gilmerspheres18}).


Another instance of our general formulation is when $\Omega = \mathcal
A^\times$, where $\mathcal A^\times := \{(x,x') \in \mathcal X^2 \mid \mathcal
A_x \cap \mathcal A_{x'} \ne \emptyset\},
$
for a system $(\mathcal A_x)_{x \in \mathcal X}$ of subsets of $\mathcal X$. This
framework is already much more general than the distance-based framework (which
is the default setting in the literature), and corresponds to the setting
considered in ~\cite{bhagoji2019}. Working at this level of generallity allows
the possibility to study general attacks like pixel-erasure
attacks~\cite{onepixel}, for example, which cannot be metrically expressed.

A \emph{type$-\Omega$ adversarial attacker} on the feature space $\mathcal X$ is
then a measurable mapping $a:\mathcal X \to \mathcal X$, such that $(x,a(x)) \in
\Omega$ for all $x \in \mathcal X$.
For example, in the distance-based attacks, it
corresponds to a measurable selection for every $x \in \mathcal X$, of some $x'
=a(x) \in \mathcal X$ with $d(x,x') \le \varepsilon$.

For $x \in \mathcal X$, denote $\Omega(x) := \{x' \in \mathcal X \mid
(x,x') \in \Omega\}$. Given a classifier $h:\mathcal X \rightarrow \{1,2\}$,
and a label $k \in \{1,2\}$, define
\begin{eqnarray}
\Omega^{h,k} := \{x \in \mathcal X \mid
h(x') \ne k\text{ for some }x' \in \Omega(x)\},
 \end{eqnarray}
 the set of examples with a "neighbor" whose predicted label is different from $k$. Conditioned on the event $Y=k$, the "size" of the set $\Omega^{h,k}$ is the
adversarial error / risk of the classifier $h$, on the class $k$. This will be made precise in the passage.
\paragraph{Adversarial Bayes-optimal error.}
  The \emph{adversarial error / risk} of a classifier $h$ 
  under type-$\Omega$ adversarial attacks, is defined by
\begin{eqnarray}
  \begin{split}
\err_{\Omega}(h; P^1, P^2) &:= \mathbb P_{X,Y}[h(x') \ne Y\text{ for some }x' \in \Omega(X)] 
=
\sum_{k=1}^2P^k(\Omega^{h,k}).
   \end{split}  
\end{eqnarray}
Thus, $\err_{\Omega}(h;P^1,P^2)$ is the least possible classification error
suffered by $h$ under type-$\Omega$ attacks.
The \emph{adversarial Bayes-optimal error}  $\err^*_{\Omega}(P^1,P^2)$ for
type-$\Omega$ attacks is defined by
  \begin{eqnarray}
    \err_{\Omega}^*(P^1,P^2) := \inf_h \err_{\Omega}(h; P^1, P^2),
    \label{eq:risk}
  \end{eqnarray}    
  where the infimum is taken over all measurable functions $h: \mathcal X
  \rightarrow \{1,2\}$, i.e over all classifiers.
  
Econometrically, adversarial Bayes-optimal error $\err_{\Omega}^*(P^1,P^2)$
corresponds to the maximal payoff of a type-$\Omega$ adversarial attacker who
tries to uniformly ``blunt'' all classifiers at the task of solving the
classification problem $(P^1,P^2)$.
  
  For the special case of distance-based attacks with budget $\varepsilon$, where the attack model is $\Omega = D_\varepsilon$ (defined in Eq. \ref{eq:epsmodel}), we will simply write
  $\err_\varepsilon^*(P^1,P^2)$ in lieu of $\err_{D_\varepsilon}(P^1,P^2)$,
  that is
  \begin{eqnarray}
  \begin{split}
    &\err_\varepsilon^*(P^1,P^2) =
    \sum_{k=1}^2P^k(D_\varepsilon^{h,k})
    =     
    \sum_{k=1}^2P^k(\{x \in \mathcal X \mid \exists x' \in \ball(x;\varepsilon),\;h(x') \ne k\}).
    \end{split}
    \label{eq:epsrisk}
  \end{eqnarray}


%% file: 05-conclusion.tex
\section{Concluding remarks}
Our results extend the current theory on the limitations of adversarial
robustness in machine learning. Using techniques from optimal transport theory,
we have obtained explicitly variational formulae and lower-bounds on the
Bayes-optimal error classifiers can attain under adversarial attack. These
formulae suggest that instead of doing adversarial training on normal data,
practitioners should strive to do normal training on adversarially augmented
data. Going further, in the case of metric attacks, we have obtained explicit
bounds which exploit the high-dimensional geometry of the class-conditional
distribution of the data. These bounds are universal in that the are
classifier-independent; they only depend on the geometric properties of the
class-conditional distribution of the data (e.g finite moments, light-tailness,
etc.).



%% file: 06-appendix.tex
\appendix

\section{Recap of main results}
For the convenience of the reader, let us begin by informally summarizing the
main contributions of our paper. Rigorus restatements (and proofs) of the
results will follow.

\input{summary.tex}
\input{060-proofs.tex}
\section{Miscellaneous}
\subsection{Computing optimal adversarial attack plan via bipartite graph matching (section \ref{subsec:matching}}

\begin{figure}[!hb]
    \centering
    \includegraphics[width=.6\linewidth]{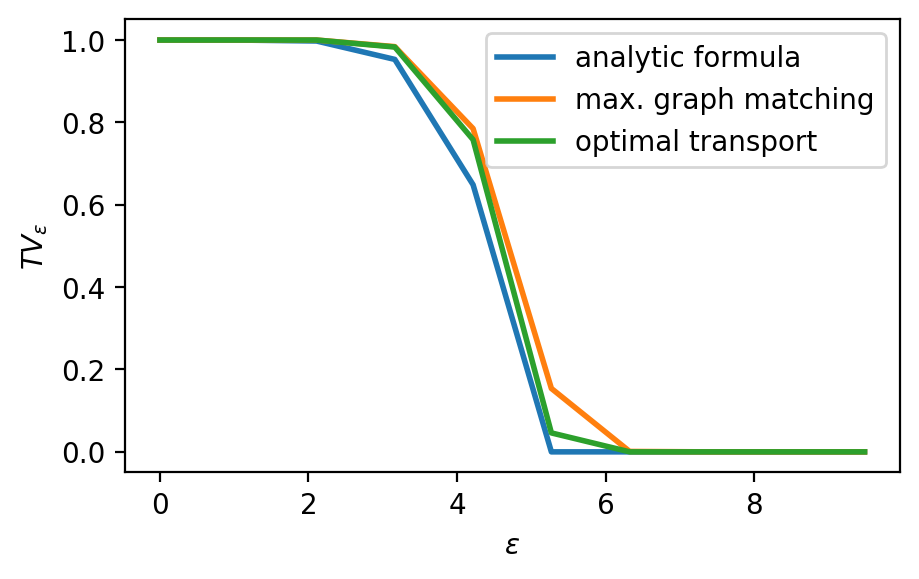}
    \label{fig:tveps_computation}
    \vspace{-.5cm}
    \caption{{\bfseries Left:} Numerical computation of
      $\tv_\varepsilon(P^1,P^2)$ for two 10-dimensional Gaussians $P^1$ and
      $P^2$ of same covariance matrix $\sigma I_{10}$ but different means. The
      maximal graph matching approach is described in Alg. \ref{alg:pv}, run on
      empirical samples from 500 iid sampes from $P^1$ and $P^2$.
    }
  \end{figure}

%% file: 060-proofs.tex
\section{Proofs of lemmas, propositions, theorems, and corollaries}
In this appendix we provide complete proofs for the theorems, corollaries,
etc. which were stated without proof in the manuscript. For clarity, each
result from the manuscript (theorems, corollaries, etc.) is restated in this
supplemental before proved.

\subsection{Proofs for results in section \ref{sec:variational}}


\main*
First note that the $\Omega$ defined is not automatically a closed subset of
$\mathcal X^2$. A sufficient condition is that the metric space $(\mathcal X,d)$
has the \emph{mid-point property}.

  Let us further suppose that
  \begin{itemize}
  \item $\Omega$ is symmetric, i.e $(x,x' ) \in \Omega$ iff
    $(x',x) \in \Omega$, and that
  \item $\Omega$ contains the diagonal of $\mathcal X^2$, i.e
    $(x,x) \in \Omega$  for all $x \in \mathcal X$.
  \end{itemize}
\begin{proof}[Proof of Proposition \ref{thm:main}]
For $x \in \mathcal X$, define $\Omega(x) := \{x' \in \mathcal X \mid (x,x') \in
\Omega\}$.
For a classifier $h$, consider the derived classifier $\widetilde{h}:\mathcal X
  \rightarrow \{1, 2, \perp\}$ defined by
  \begin{eqnarray}
    \widetilde{h}(x) :=
    \begin{cases}
      y,&\mbox{ if } h(x') = y\; \forall x' \in \Omega(x),\\ \perp,&\mbox{ else.}
    \end{cases}
  \end{eqnarray}
Here, the special symbol $\perp \not\in \{1, 2\}$ should be read as ``I don't
know!''. Let $X_1$ (resp. $X_2$) be a random variable that has the same
distribution as $X$ conditioned on the event $Y=1$ (resp. $Y=2$). One easily
computes $1-\err_{\Omega}(h; P^1,P^2) =
\frac{1}{2}\mathbb P(\widetilde{h}(X_1)=1)+\frac{1}{2}\mathbb P(\widetilde{h}(X_2)=2)$, from which
  \begin{eqnarray}
    2(1-\err_{\Omega}(h; P^1,P^2)) = \mathbb E[\mathbbm{1}[\widetilde{h}(X_1)=1]]+\mathbb
    E[\mathbbm{1}[\widetilde{h}(X_2)=2]].
    \label{eq:trick}
  \end{eqnarray}
  Now, define $g_0(x'):=\mathbbm{1}[\widetilde{h}(x')=1]$ and
  $f_0(x)=\mathbbm{1}[\widetilde{h}(x) \ne 2]=1-\mathbbm{1}[\widetilde{h}(x)=2]$.
  Then $f_0$ and $g_0$ are bounded, and $P^2$- (resp. $P^1$-) a.s continuous.
  Moreover, given $x', x \in \mathcal X$, if $c_{\Omega}(x',x)=1$, then
  $x' \not\in \Omega(x)$. Since
  $\{h^{-1}(\{y\})\}_{y=1, 2}$ is a partitioning of $\mathcal X$, at most one of
  $\Omega(x') \subseteq h^{-1}(\{1\})$ and $\Omega(x) \subseteq
  h^{-1}(\{2\})$ holds.
  Thus $\mathbbm{1}[\widetilde{h}(x')=1] + \mathbbm{1}[\widetilde{h}(x)=2] \le 1$, and so
  $$g_0(x')-f_0(x)=\mathbbm{1}[\widetilde{h}(x')=1] + \mathbbm{1}[\widetilde{h}(x')=2] - 1 \le
  c_{\Omega}(x',x),
  $$
  and so $(f_0, g_0)$ is a pair of Kantorovich potentials for the
  cost function $c_{\Omega}$.
  Consequently, from the Kantorovich-Rubinstein duality
  formula, we have
  \begin{eqnarray*}
    \begin{split}
    \text{OT}_{\Omega}(P^2,P^1)& = \sup_{\text{potentials } f, g}\mathbb
    E[g(X_1)]-\mathbb E[f(X_2)]
    \ge \mathbb
    E[g_0(X_1)]-\mathbb E[f_0(X_2)]\\
    &= \mathbb
    E[\mathbbm{1}[\widetilde{h}(X_1)=1]]+\mathbb E[\mathbbm{1}[\widetilde{h}(X_2)=2]]2
    \\
    &\overset{\eqref{eq:trick}}{=}
    2(1-\err_{\Omega}(h; P^1,P^2))
    =1-2\err_{\Omega}(h; P^1,P^2).
    \end{split}
  \end{eqnarray*}
 Since $h$ is an arbitrary classifier, we obtain that
 $2\;\underset{h}{\inf}\;\err_{\Omega}(h; P^1,P^2) \ge 1-\text{OT}_{\Omega}(P^2,P^1)$ as claimed.
\end{proof}

\strassen*
\begin{proof}
One computes
\begin{eqnarray}
\begin{split}
\ot_{\Omega}(P^1,P^2) &:= \inf_{\gamma \in \Pi(P^1,P^2)}\int_{\mathcal X^2}
c_{\Omega}(x,x')\diff\gamma(x,x') = \inf_{\gamma \in
  \Pi(P^1,P^2)}\int_{\mathcal X^2} \mathbbm{1}[(x,x') \in \Omega]
\diff\gamma(x,x')\\ &=\inf_{\gamma \in \Pi(P^1,P^2)}\int_{\Omega}
\diff\gamma(x,x')= \inf_{\gamma \in \Pi(P^1,P^2)}\gamma(\Omega),
\end{split}
\label{eq:prestrassen}
\end{eqnarray}
On the other hand, by \emph{Strassen's Marriage Theorem}  (see ~\cite[Theorem
1.27 of]{villaniTopics}) and the definition of $\overline{U}_\Omega$ in Eq.
\eqref{eq:Uneb}, one has
$$
\inf_{\gamma \in \Pi(P^1,P^2)}\gamma(\Omega)=\underset{ U \subseteq \mathcal
  X\;\text{closed}}{\sup}\; P^2(U)-P^1(\overline{U}_{\Omega}),
$$
and the result follows. The particular case of distance-based attacks
corresponds to letting $\Omega := D_\varepsilon := \{(x,x') \in \mathcal X^2
\mid d(x,x') \le \varepsilon\}$, so that $\overline{U}_\Omega = U^\varepsilon$,
the $\varepsilon$-neighborhood of $U$.
\end{proof}


\badass*

\begin{proof}
Note that each $P^k$ is a Borel measure on $\mathcal X$ which integrates
to $1/2$. For the convenience of the proof, we rescale each $P^k$ by $2$, so
that it integrates to $1$.

Let $D'_\varepsilon := \mathcal X^2 \setminus D_\varepsilon$. To prove the
theorem, we consider the following intermediate quantity
$$
E(P^1,P^2):=\inf_{\gamma\in \Pi_{\le}(P^1,P^2),\; \supp(\gamma)\subseteq
  D_\varepsilon} 1 - \gamma(\mathcal{X}^2).
$$
Applying Lemma \ref{lm:partial} with $\Omega = D_\varepsilon$, we know that
$\ot_{\varepsilon} = E$. The rest of the proof is divided into separate steps.

\paragraph{Step 1: proving the equality $E=\widetilde{\tv}_{\varepsilon}$.} Let $\gamma$
be feasible for $E$. Because $(\mathcal X,d)$\ satisfies the MM property (Condition
\ref{cond:mmp}), there exists a $\gamma$-measurable map $\eta:\mathcal X^2 \to
\mathcal X$ such that $\eta(x,x')$ is a midpoint of $x$ and $x'$ for all $x,x' \in
\mathcal X$. Now, consider the $\gamma$-measurable maps $T_1,T_2:\mathcal X^2
\rightarrow \mathcal X^2$, $D:\mathcal X \rightarrow \mathcal X^2$ defined by
\begin{eqnarray}
\begin{split}
T_1(x,x') := (x,\eta(x,x')),\;
T_2(x,x') := (\eta(x,x'),x'),\;T_3(x)=(x,x).
\end{split}
\end{eqnarray}
Construct couplings $\gamma_1 = (T_1)_\#
\gamma + {T_3}_\# (P^1-\proj^1_\#\gamma)$ and $\gamma_2 = (T_2)_\# \gamma +
{T_3}_\# (P^2-\proj^2_\#\gamma)$. Then $(\gamma_1,\gamma_2)$ is feasible for
$\widetilde{\tv}_\epsilon$ and
\begin{eqnarray*}
\begin{split}
\tv(\proj^2_\#\gamma_1,\proj^1_\#\gamma_2)
\leq (P^1-\proj^1_\#\gamma)(\mathcal{X})+ (P^2-\proj^2_\#\gamma)(\mathcal{X})
= 2(1-\gamma(\mathcal{X}^2))
\end{split}
\end{eqnarray*}
because the second marginal of $(T_1)_\# \gamma$ and the first  marginal of
$(T_2)_\# \gamma$ agree by construction. Thus $\tv_{\varepsilon}\leq E$.
Conversely, let $\gamma_1,\gamma_2$ be feasible for $\widetilde{\tv}_{\varepsilon}$,
and let $\tilde \gamma_1\leq \gamma_1$ and $\tilde \gamma_2\leq \gamma_2$ be
such that $\proj^2_\#\tilde \gamma_1 = \proj^1_\#\gamma_2 = (\proj^2_\#
\gamma_1)\wedge (\proj^1_\# \gamma_2)$ where $\wedge$ is the "pointwise" minimum
of two measures (they can be built with the disintegration theorem). Now build
$\tilde \gamma$ feasible for $E$ by gluing together $\tilde \gamma_1$ and
$\tilde \gamma_2$. It holds
$$
\tv(\proj^2_\#\gamma_1 ,\proj^1_\# \gamma_2 ) = 2(1-\proj^2_\#\gamma_1  \wedge
\proj^1_\# \gamma_2)(\mathcal{X}) = 2(1-\tilde \gamma(\mathcal{X}^2)).
$$
Thus $E\leq \widetilde{\tv}_{\varepsilon}$ hence $E=\widetilde{\tv}_{\varepsilon}$.

\paragraph{Step 2: proving the inequality $\widetilde{\tv}_{\varepsilon}\leq
  \widetilde{\widetilde{\tv}}_{\varepsilon}$.} Now, the fact that $\widetilde{\tv}_{\varepsilon}\leq
\widetilde{\widetilde{\tv}}_{\varepsilon}$ in general is due to the fact to any
transport map $a$ satisfying $d(a(x),x) \leq \varepsilon$ corresponds a
deterministic transport plan $(\mathrm{id}, a)_\# P^1$ supported on
$D_\varepsilon$. In general, equality in the theorem will fail to hold. For
example, on the real line, $P^1=\frac13 \delta_{-\varepsilon}+\frac13
\delta_{0}+\frac13 \delta_{\varepsilon}$ and $P^2=\frac12
\delta_{-\varepsilon}+\frac12 \delta_\varepsilon$ has
$\widetilde{\widetilde{\tv}}_{\varepsilon}(P^1,P^2)=2/6$ and
$\widetilde{\tv}_{\varepsilon}(P^1,P^2)=0$.

\paragraph{Step 3: proving the inequality
  $\widetilde{\tv}_{\varepsilon} \ge \widetilde{\widetilde{\tv}}_{\varepsilon}$ for
  absolutely continuous $P^k$'s.} Finally, the fact that
$\widetilde{\tv}_{\varepsilon} \ge \widetilde{\widetilde{\tv}}_{\varepsilon}$ when
$P^1$ and $P^2$ are absolutely continuous is a consequence of the existence of
an optimal transport map for the $W_\infty$ distance. 
Indeed, if $(\gamma_1,\gamma_2)$ is feasible for
$\widetilde{\tv}_{\varepsilon}$, then $W_\infty(P^1,\proj^2_\#\gamma_1)\leq
\varepsilon$ and there exists a measurable map $a_1:\mathcal{X}\to\mathcal{X}$
such that $d(a_1,x)\leq \varepsilon$ $P^1$-a.e. and $(a_1)_\# P^1 =
\proj^2_\#\gamma_1$ (one can build $a_2$ similarly).
\end{proof}

\begin{proof}
The first part of the claim follows from a direct application of ~\cite[Theorem 1]{barsov87}:
  \begin{eqnarray*}
  \begin{split}
\tv(\mathcal N(\mu_1,\Sigma),\mathcal N(\mu_2,\Sigma))=
  2\Phi(\|\mu\|_{\Sigma^{2}}/2)2,   
  \end{split}
  \end{eqnarray*}
  where $\mu:=\mu_1-\mu_1 \in \mathbb R^d$.
Thus  $\err_{\Omega}^* \ge 1- \Phi(\Delta(\varepsilon)/2)$, where
$\Delta(\varepsilon):=\min_{\|z\|_{\mathcal X} \le
  \varepsilon}\|z-\mu\|_{\Sigma^{2}}$. It now remains to bound
$\Phi(\Delta(\varepsilon))$, and we are led to consider the computation of
quantities of the following form.

\input{cool.tex}
\end{proof}

%% file: cool.tex
\paragraph{Bounding the quantity $\Delta(\varepsilon)$.}
We are led to consider problems of the form
\begin{eqnarray}
\alpha &:=\max_{\|w\|_\Sigma \le 1} w^Ta-\varepsilon\|w\|_1,
\end{eqnarray}
where $a \in \mathbb R^d$ and $\Sigma$ be a positive definite matrix of size $d$.
Of course, the solution value might not be analytically expressible in general, but there is some hope, when the matrix $\Sigma$ is diagonal. That notwithstanding, using the dual representation of the $\ell_1$-norm, one has
\begin{eqnarray}
\begin{split}
\alpha &= \max_{\|w\|_\Sigma \le 1}\min_{\|z\|_\infty \le \varepsilon}w^T a-w^Tz
=\max_{\|z\|_\infty \le \varepsilon}\min_{\|w\|_\Sigma \le 1}w^T(z-a)\\
&=\min_{\|z\|_\infty \le \varepsilon}\left(\max_{\|w\|_\Sigma \le 1}w^T(z-a)\right)
=\min_{\|z\|_\infty \le \varepsilon}\left(\max_{\|\tilde{w}\|_2 \le 1}\tilde{w}^T\Sigma^{-1}(z-a)\right)\\
&=\min_{\|z\|_\infty \le \varepsilon}\|z-a\|_{\Sigma^{-1}}
=\min_{\|z\|_\infty \le \varepsilon}\|z-a\|_{\Sigma^{-1}},
\end{split}
\label{eq:alpha}
\end{eqnarray}
where we have used \emph{Sion's minimax theorem} to interchange min and max in the first line, and we have introduced the auxiliary variable $\tilde{w}:=\Sigma^{-1/2}w$ in the fourth line. We note that given a value for the dual variable $z$, the optimal value of the primal variable $w$ is
\begin{eqnarray}
  w \propto \frac{\Sigma^{-1}(a-z)}{\|\Sigma^{-1}(a-z)\|_2}.
  \label{eq:wopt}
\end{eqnarray}
The above expression ~\eqref{eq:alpha} for the optimal objective value $\alpha$ is unlikely to be computable analytically in general, due to the non-separability of the objective (even though the constraint is perfectly separable as a product of 1D constraints).
In any case, it follows from the above display that $\alpha \le 0$, with equality iff $\|a\|_\infty \le \varepsilon$.
\qed

\paragraph{Exact formula for diagonal $\Sigma$.}
In the special case where $\Sigma=\text{diag}(\sigma_1,\ldots,\sigma_2)$, the square of the optimal objective value $\alpha^2$ can be separated as
\begin{eqnarray*}
\begin{split}
\alpha \ge 0,\;\alpha^2 = \sum_{j=1}^d\min_{|z_j| \le \varepsilon}\sigma_j^{-2}(z_j-a_j)^2
&= \sum_{j=1}^d\sigma_j^{-2}\begin{cases}(a_j+\varepsilon)^2,&\mbox{ if }a_j \le -\varepsilon,\\ 0,&\mbox{ if }-\varepsilon < a_j \le \varepsilon,\\ (a_j-\varepsilon)^2,&\mbox{ if }a_j > \varepsilon,\end{cases}\\
&=\sum_{j=1}^d\sigma_j^{-2}((|a_j|-\varepsilon)_+)^2.
\end{split}
\end{eqnarray*}
Thus $\alpha = \sqrt{\sum_{j=1}^d\sigma_j^{-2}((|a_j|-\varepsilon)_+)^2}$.
By the way, the optimium is attained at
\begin{eqnarray}
  \begin{split}
  z_j &= \begin{cases}
    -\varepsilon,&\mbox{ if }a_j \le -\varepsilon,\\ a_j,&\mbox{ if }-\varepsilon  < a_j \le \varepsilon,\\
    \varepsilon,&\mbox{ if }a_j > \varepsilon,    
\end{cases}\\
&=a_j-\sign(a_j)(|a_j|-\varepsilon)_+
  \end{split}
\end{eqnarray}
Plugging this into \eqref{eq:wopt} yields the optimal weights
\begin{eqnarray}
w_j \propto \sigma^{-2}_j\sign(a_j)(|a_j|-\varepsilon)_+.
\end{eqnarray}

\paragraph{Upper / lower bounds for general $\Sigma$.}
Let $\sigma_1,\sigma_2,\ldots,\sigma_d > 0$ be the eigenvalues of $\Sigma$. Then 
$$
\|z-a\|_{\Sigma^{-1}}^2 := (z-a)^T\Sigma^{-1}(w-a) \le \sum_{j=1}^d(z_j-a_j)^2/\sigma_j^2 =: \|z-a\|_{\diag(1/\sigma_1,\ldots,1/\sigma_d)}^2.
$$
Therefore in view of the previous computations for diagonal covariance matrices, one has the bound $\alpha \le \sqrt{\sum_{j=1}^d\sigma_j^{-2}((|a_j|-\varepsilon_j)_+)^2}$.